\documentclass[nohyperref]{article}

\usepackage{microtype}
\usepackage{graphicx}
\usepackage{subfigure}
\usepackage{booktabs} 
\usepackage{thm-restate}
\usepackage{bm}

\usepackage{hyperref}

\usepackage[accepted]{icml2022}

\usepackage{amsmath}
\usepackage{amssymb}
\usepackage{mathtools}
\usepackage{amsthm}

\usepackage[capitalize,noabbrev]{cleveref}

\theoremstyle{plain}
\newtheorem{theorem}{Theorem}[section]
\newtheorem{proposition}[theorem]{Proposition}
\newtheorem{lemma}[theorem]{Lemma}
\newtheorem{corollary}[theorem]{Corollary}
\theoremstyle{definition}

\newtheorem{problem}[theorem]{Problem}
\newtheorem{fact}[theorem]{Fact}
\theoremstyle{remark}
\newtheorem{remark}[theorem]{Remark}

\newcommand{\width}{\operatornamewithlimits{width}}
\newcommand{\argmax}{\operatornamewithlimits{argmax}}
\newcommand{\argmin}{\operatornamewithlimits{argmin}}

\newcommand{\E}{\mathbb{E}}

\usepackage[textsize=tiny]{todonotes}

\icmltitlerunning{Thompson Sampling for (Combinatorial) Pure Exploration}

\allowdisplaybreaks

\begin{document}

\twocolumn[
\icmltitle{Thompson Sampling for (Combinatorial) Pure Exploration}



\icmlsetsymbol{equal}{*}

\begin{icmlauthorlist}
\icmlauthor{Siwei Wang}{tsinghua}
\icmlauthor{Jun Zhu}{tsinghua}
\end{icmlauthorlist}

\icmlaffiliation{tsinghua}{Dept. of Comp. Sci. \& Tech., BNRist Center, Tsinghua-Bosch Joint ML Center, Tsinghua University}
\icmlcorrespondingauthor{Jun Zhu}{dcszj@tsinghua.edu.cn}

\icmlkeywords{Combinatorial Pure Exploration, Thompson Sampling}

\vskip 0.3in
 ]



\printAffiliationsAndNotice{}  

\begin{abstract}
%
Existing methods of combinatorial pure exploration mainly focus on the UCB approach.
To make the algorithm efficient, they usually use the sum of upper confidence bounds within arm set $S$ to represent the upper confidence bound of $S$, which can be much larger than the tight upper confidence bound of $S$ and leads to a much higher complexity than necessary, since the empirical means of different arms in $S$ are independent.
%
%
%
To deal with this challenge, we explore the idea of Thompson Sampling (TS) that uses independent random samples instead of the upper confidence bounds, 
and design the first TS-based algorithm TS-Explore for (combinatorial) pure exploration.
%
%
In TS-Explore, the sum of independent random samples within arm set $S$ will not exceed the tight upper confidence bound of $S$ with high probability.
Hence it solves the above challenge, and achieves a lower complexity upper bound than existing efficient UCB-based algorithms in general combinatorial pure exploration. %
As for pure exploration of classic multi-armed bandit, we show that TS-Explore achieves an asymptotically optimal complexity upper bound.
\end{abstract}

\section{Introduction}


Pure exploration is an important task in online learning, and it tries to find out the target arm as fast as possible. 
%
%
In pure exploration of classic multi-armed bandit (MAB) \citep{Audibert2010Best}, there are totally $m$ arms, and each arm $i$ is associated with a probability distribution $D_i$ with mean $\mu_i$. 
Once arm $i$ is pulled, it returns an observation $r_i$, which is drawn independently from $D_i$ by the environment. 
At each time step $t$, the learning policy $\pi$ either chooses an arm $i(t)$ to pull, or chooses to output an arm $a(t)$. 
The goal of the learning policy is to pull arms properly, such that with an error probability at most $\delta$, its output arm $a(t)$ is the optimal arm (i.e., $a(t) = \argmax_{i \in [m]} \mu_i$) with complexity (the total number of observations) as low as possible.



Pure exploration is widely adopted in real applications. 
For example, in the selling procedure of cosmetics, there is always a testing phase before the commercialization phase \citep{Audibert2010Best}. 
The goal of the testing phase is to help to maximize the cumulative reward collected in the commercialization phase. 
Therefore, instead of regret minimization \citep{Berry1985Bandit,Auer2002Finite}, the testing phase only needs to do exploration (e.g., to investigate which product is the most popular one), and wants to find out the target with both correctness guarantee and low cost.
%
In the real world, sometimes the system focuses on the best action under a specific combinatorial structure, instead of the best single arm \citep{Chen2014Combinatorial}. 
For example, a network routing system needs to search for the path with minimum delay between the source and the destination. 
Since there can be an exponential number of paths, the cost of exploring them separately is unacceptable. 
Therefore, people choose to find out the best path by exploring single edges, and this is a pure exploration problem instance in combinatorial multi-armed bandit (CMAB).
In this setting, we still pull single arms (base arms) at each time step, and there is a super arm set $\mathcal{I} \subseteq 2^{[m]}$. 
The expected reward of a super arm $S\in \mathcal{I}$ is $\sum_{i\in S} \mu_i$, i.e., the sum of the expected rewards of its contained base arms. 
And the goal of the player is to find out the optimal super arm with error probability at most $\delta$ and complexity as low as possible. 



Most of the existing solutions for pure exploration follow the UCB approach \citep{Audibert2010Best,kalyanakrishnan2012PAC,Chen2014Combinatorial,kaufmann2013information}. They compute the confidence bounds for all the arms, and claim that one arm is optimal only if its lower confidence bound is larger than the upper confidence bounds of all the other arms. 
Though this approach is asymptotically optimal in pure exploration of classic MAB problems \citep{kalyanakrishnan2012PAC,kaufmann2013information}, it faces some challenges in the CMAB setting \citep{Chen2014Combinatorial}. 
In the UCB approach, for a super arm $S$, the algorithm usually uses the sum of upper confidence bounds of all its contained base arms as its upper confidence bound to achieve a low implementation cost.
%
%
This means that the gap between the empirical mean of super arm $S$ and the used upper confidence bound of $S$ is about $\tilde{O}(\sum_{i\in S} \sqrt{1/N_i})$ (here $N_i$ is the number of observations on base arm $i$).
However, since the observations of different base arms are independent, the standard deviation of the empirical mean of super arm $S$ is $\tilde{O}(\sqrt{\sum_{i\in S} {1/N_i}})$, which is much smaller than $\tilde{O}(\sum_{i\in S} \sqrt{1/N_i})$.
This means that the used upper confidence bound of $S$ is much larger than the tight upper confidence bound of $S$.
\citet{combes2015combinatorial} deal with this problem by computing the upper confidence bounds for all the super arms independently, %
which leads to an exponential time cost and is hard to implement in practice. 
%
%
%
In fact, existing efficient UCB-based solutions either suffer from a high complexity bound \citep{Chen2014Combinatorial,gabillon2016improved}, or need further assumptions on the combinatorial structure 
to achieve an optimal complexity bound efficiently \citep{chen2017nearly}.

Then a natural solution to deal with this challenge is to use random samples of arm $i$ (with mean to be its empirical mean and standard deviation to be $\tilde{O}(\sqrt{1/ N_i})$) instead of its upper confidence bound to judge whether a super arm is optimal. 
If we let the random samples of different base arms be \emph{independent}, then with high probability, the gap between the empirical mean of super arm $S$ and the sum of random samples within $S$ is $\tilde{O}(\sqrt{\sum_{i\in S} {1/ N_i}})$, which has the same order as the gap between the empirical mean of super arm $S$ and its real mean.
Therefore, using independent random samples can 
behave better than using confidence bounds,
and this is the key idea of Thompson Sampling (TS) \citep{thompson1933likelihood,Kaufmann2012Thompson,agrawal2013further}. 
In fact, many prior works show that TS-based algorithms have smaller cumulative regret than UCB-based algorithms in regret minimization of CMAB model \citep{WC18,perrault2020statistical}. 
However, there still lack studies on adapting the idea of TS, i.e., using random samples to make decisions, to the pure exploration setting.
%

In this paper, we attempt to fill up this gap, and study using random samples in pure exploration under the frequentist setting for both MAB and CMAB instances. 
We emphasize that it is non-trivial to design (and analyze) such a TS-based algorithm. 
The first challenge is that there is a lot more uncertainty in the random samples than in the confidence bounds.
In UCB-based algorithms, the upper confidence bound of the optimal arm is larger than its expected reward with high probability.
Thus, the arm with the largest upper confidence bound is either an insufficiently learned sub-optimal arm (i.e., the number of its observations is not enough to make sure that it is a sub-optimal arm) or the optimal arm, which means that the number of pulls on sufficiently learned sub-optimal arms is limited.
However, for TS-based algorithms that use random samples instead, there is always a constant probability that the random sample of the optimal arm is smaller than its expected reward.
In this case, the arm with the largest random sample may be a sufficiently learned sub-optimal arm. 
Therefore, 
the mechanism of the TS-based policy should be designed carefully to make sure that we can still obtain an upper bound for the number of pulls on sufficiently learned sub-optimal arms. 

Another challenge is that using random samples to make decisions is a kind of Bayesian algorithm and it loses many good properties in the frequentist setting. 
%
%
%
%
In the Bayesian setting, at each time step $t$, the parameters of the game follow a posterior distribution $\mathcal{P}_t$, and the random samples are drawn from $\mathcal{P}_t$ independently as well. 
Therefore, using the random samples to output the optimal arm can explicitly ensure the correctness of the TS-based algorithm.
However, in the frequentist setting, the parameters of the game are fixed but unknown, and they have no such correlations with the random samples. 
Because of this, if we still choose to use the random samples to output the optimal arm, then the distributions of random samples in the TS-based algorithm need to be chosen carefully to make sure that we can still obtain the correctness guarantee in the frequentist setting.

Besides, the analysis of the TS-based algorithm in pure exploration is also very different from that in regret minimization. 
In regret minimization, at each time step, we only need to draw \emph{one} sample for each arm \cite{thompson1933likelihood,agrawal2013further,WC18,perrault2020statistical}. 
However, in pure exploration, one set of samples at each time step is not enough, since the algorithm needs to i) check whether there is an arm that is the optimal one with high probability; ii) look for an arm that needs exploration.
None of these two goals can be achieved by only \emph{one} set of samples. 
Therefore, we must draw several sets of samples to make decisions, and this may fail the existing analysis of TS in regret minimization.

%
%
%
%


%
In this paper, we solve the above challenges, and design a TS-based algorithm TS-Explore for (combinatorial) pure exploration under the frequentist setting 
%
%
with polynomial implementation cost.  
At each time step $t$, TS-Explore first draws independent random samples $\theta_i^k(t)$ for all the (base) arms $i\in [m]$ and $k = 1,2,\cdots,M$ (i.e., totally $M$ independent samples for each arm). 
Then it tries to find out the $M$ best (super) arms $\tilde{S}^k_t$'s under sample sets $\bm{\theta}^k(t) = [\theta_1^k(t), \theta_2^k(t), \cdots,\theta_m^k(t)]$ for $k = 1,2,\cdots,M$. 
If in all these sample sets, the best (super) arm is the same as the empirically best (super) arm $\hat{S}_t$ (i.e., the best arm under the empirical means), then the algorithm will output that this (super) arm is optimal. 
Otherwise, for all $k \in [M]$, the algorithm will check the reward gap between $\tilde{S}^k_t$ and $\hat{S}_t$ under parameter set $\bm{\theta}^k(t)$.
Then it focuses on the (super) arm $\tilde{S}^{k_t^*}_t$ with the largest reward gap, i.e., it chooses to pull a (base) arm in the exchange set of $\hat{S}_t$ and $\tilde{S}^{k_t^*}_t$. 

Recording such reward gaps and focusing on $\tilde{S}^{k_t^*}_t$ is the key mechanism we used to solve the above challenges. 
On the one hand, our analysis shows that with a proper value $M$ (the number of sample sets) and proper random sample distributions, $\tilde{S}^{k_t^*}_t$ has some similar properties with the (super) arm with the largest upper confidence bound.
These properties play a critical role in the analysis of UCB approaches.
%
%
Thus, we can also use them to obtain an upper bound for the number of pulls on sufficiently learned sub-optimal arms in TS-Explore as well as the correctness guarantee of TS-Explore.
On the other hand, this novel mechanism saves the key advantage of Thompson Sampling, i.e., the sum of random samples within $S$ will not exceed the tight upper confidence bound of $S$ with high probability.
Hence, TS-Explore can achieve a lower complexity upper bound than existing efficient UCB-based solutions.
%
%
%
%
%
These results indicate that our TS-Explore policy is correct, efficient (with low implementation cost), and effective (with low complexity).

%
%
%
%
%
%
%
%
%
In the general CMAB setting, we show that TS-Explore is near optimal and achieves a lower complexity upper bound than existing efficient UCB-based algorithms \citep{Chen2014Combinatorial}.
The optimal algorithm in \citep{chen2017nearly} is only efficient when the combinatorial structure satisfies some specific properties (otherwise it suffers from an exponential implementation cost),  
and is less general than our results.
We also conduct experiments to compare the complexity of TS-Explore with existing baselines.
The experimental results show that TS-Explore outperforms the efficient baseline CLUCB \cite{Chen2014Combinatorial}, and behaves only a little worse than the optimal but non-efficient baseline NaiveGapElim \cite{chen2017nearly}. 
%
As for the MAB setting, we show that TS-Explore is asymptotically optimal, i.e., it has a comparable complexity to existing optimal algorithms \citep{kalyanakrishnan2012PAC,kaufmann2013information} when $\delta \to 0$.
All these results indicate that our TS-based algorithm is efficient and effective in dealing with pure exploration problems.
%
%
%
%
%
To the best of our knowledge, this is the first result of using this kind of TS-based algorithm (i.e., always making decisions based on random samples) in pure exploration under the frequentist setting.

\section{Related Works}


Pure exploration of the classic MAB model is first proposed by \citet{Audibert2010Best}.
After that, people have designed lots of learning policies for this problem. 
The two most representative algorithms are successive-elimination \citep{even2006action,Audibert2010Best,kaufmann2013information} and LUCB \citep{kalyanakrishnan2012PAC,kaufmann2013information}.
Both of them adopt the idea of UCB \citep{Auer2002Finite} and achieve an asymptotically optimal complexity upper bound (i.e., it matches with the complexity lower bound proposed by \citet{kalyanakrishnan2012PAC}). 
Compared to these results, our TS-Explore policy uses a totally different approach, and can achieve an asymptotically optimal complexity upper bound as well.



Combinatorial pure exploration is first studied by \citet{Chen2014Combinatorial}. 
They propose CLUCB, an LUCB-based algorithm that is efficient as long as there is an offline oracle to output the best super arm under any given parameter set. 
%
%
\citet{chen2017nearly} then design an asymptotically optimal algorithm for this problem. However, their algorithm can 
only be implemented efficiently when the combinatorial structure follows some specific constraints.
%
%
Recently, based on the game approach, \citet{jourdan2021efficient} provide another optimal learning policy for pure exploration in CMAB. But their algorithm still suffers from an exponential implementation cost.
%
Compared with these UCB-based algorithms, our TS-Explore policy achieves a lower complexity bound than CLUCB \citep{Chen2014Combinatorial} (with a similar polynomial time cost), and  has a much lower implementation cost than the optimal policies \cite{chen2017nearly,jourdan2021efficient} in the most general combinatorial pure exploration setting. 
%
%

There also exist some researches about applying Thompson Sampling to pure exploration. For example, \citet{russo2016simple} considers a frequentist setting of pure exploration in classic MAB, and proposes algorithms called TTPS, TTVS, and TTTS; \citet{shang2020fixed} extend the results in \cite{russo2016simple}, design a T3C algorithm, and provide its analysis for Gaussian bandits; \citet{li2021bayesian} study Bayesian contextual pure exploration, propose an algorithm called BRE and obtain its corresponding analysis. However, these results are still very different from ours. The first point is that they mainly use random distributions but not random samples to 
decide the next chosen arm or when to stop, %
%
%
which may cause a high implementation cost when we extend them to the combinatorial setting, since it is much more complex to deal with random distributions than to deal with random samples. 
%
%
%
%
Moreover, our results are still more general even if we only consider pure exploration in classic MAB under the frequentist setting.
For example, \citet{russo2016simple} does not provide a correct stopping rule, and \citet{shang2020fixed} only obtain the correctness guarantee for Gaussian bandits. Besides, their complexity bounds are asymptotic ones (which require $\delta \to 0$), while ours works for any $\delta \in (0,1)$.

\section{Model Setting}

\subsection{Pure Exploration in Multi-armed Bandit}


A pure exploration problem instance of MAB is a tuple $([m], \bm{D}, \delta)$. 
Here $[m] = \{1,2,\cdots,m\}$ is the set of arms, $\bm{D} = \{D_1, D_2, \cdots, D_m\}$ are the corresponding reward distributions of the arms, and $\delta$ is the error constraint. 
In this paper, we assume that all the distributions $D_i$'s are supported on $[0,1]$. 
Let $\mu_i \triangleq \E_{X \sim D_i}[X]$ denote the expected reward of arm $i$, and $a^* = \argmax_{i\in [m]} \mu_i$ is the optimal arm with the largest expected reward.
Similar to many existing works (e.g., \citep{Audibert2010Best}), we assume that the optimal arm is unique.
At each time step $t$, the learning policy $\pi$ can either pull an arm $i(t) \in [m]$, or output an arm $a(t) \in [m]$. 
If it chooses to pull arm $i(t)$, then it will receive an observation $r_{i(t)}(t)$, which is drawn independently from $D_{i(t)}$. 
The goal of the learning policy is to make sure that with probability at least $1-\delta$, its output $a(t) = a^*$. 
Under this constraint, it aims to minimize the complexity 
\begin{equation*}
     Z^{\pi} \triangleq \sum_{i=1}^m N_i(T^{\pi}),
\end{equation*}
where $T^{\pi}$ denotes the time step $t$ that policy $\pi$ chooses to output $a(t)$, and $N_i(t)$ denotes the number of observations on arm $i$ until time step $t$.

Let $\Delta_{i,m} \triangleq \mu_{a^*} - \mu_i$ denote the expected reward gap between the optimal arm $a^*$ and any other arm $i\ne a^*$. 
For the optimal arm $a^*$, its $\Delta_{a^*,m}$ is defined as $\mu_{a^*} - \max_{i\ne a^*} \mu_i$.
We also define $H_{m} \triangleq \sum_{i\in [m]} {1\over \Delta_{i,m}^2}$, and existing works \citep{kalyanakrishnan2012PAC} show that the complexity lower bound of any pure exploration algorithm is $\Omega(H_{m} \log{1\over \delta})$.


\subsection{Pure Exploration in Combinatorial Multi-armed Bandit}\label{Section_Model_CMAB}


A pure exploration problem instance of CMAB is an extension of the MAB setting. 
%
The arms $i\in [m]$ are called base arms, and there is also a super arm set $\mathcal{I} \subseteq 2^{[m]}$.
For each super arm $S \in \mathcal{I}$, its expected reward is $\sum_{i\in S} \mu_i$.
Let $S^* = \argmax_{S\in \mathcal{I}} \sum_{i\in S} \mu_i$ denote the optimal super arm with the largest expected reward, and we assume that the optimal super arm is unique (as in \citep{Chen2014Combinatorial}).
At each time step $t$, the learning policy $\pi$ can either pull a base arm $i(t) \in [m]$, or output a super arm $S(t) \in \mathcal{I}$. 
%
%
The goal of the learning policy is to make sure that with probability at least $1-\delta$, its output $S(t) = S^*$. Under this constraint, it also wants to minimize its complexity $Z^{\pi}$.

As in many existing works \citep{Chen2013Combinatorial,Chen2014Combinatorial,WC18}, we also assume that there exists an offline ${\sf Oracle}$, which takes a set of parameters $\bm{\theta} = [\theta_1, \cdots, \theta_m]$ as input, and outputs the best super arm under this parameter set, i.e., ${\sf Oracle}(\bm{\theta}) = \argmax_{S\in \mathcal{I}} \sum_{i\in S}\theta_i$.

In this paper, for $i \notin S^*$, we use $\Delta_{i,c} \triangleq \sum_{j\in S^*} \mu_j - \max_{S \in \mathcal{I}: i\in S} \sum_{j\in S} \mu_j$ to denote the expected reward gap between the optimal super arm $S^*$ and the best super arm that contains $i$. 
As for $i \in S^*$, its $\Delta_{i,c}$ is defined as $\Delta_{i,c} \triangleq \sum_{j\in S^*} \mu_j - \max_{S \in \mathcal{I} : i\notin S} \sum_{j\in S} \mu_j$, i.e., the expected reward gap between $S^*$ and the best super arm that does not contain $i$.
We also define $S \oplus S' = (S\setminus S') \cup (S'\setminus S)$ and $\width \triangleq \max_{S\ne S'} |S \oplus S'|$, and let $H_{1,c} \triangleq \width \sum_{i\in [m]} {1\over \Delta_{i,c}^2}$, $H_{2,c} \triangleq \width^2 \sum_{i\in [m]} {1\over \Delta_{i,c}^2}$.

\citet{chen2017nearly} prove that the complexity lower bound for combinatorial pure exploration is $\Omega(H_{0,c}\log{1\over \delta})$, where $H_{0,c}$ is the optimal value of the following convex program (here $\Delta_{S^*,S} = \sum_{i\in S^*} \mu_i - \sum_{i\in S} \mu_i$):
\begin{eqnarray*}
 \min && \sum_{i\in [m]} N_i\\
 \mathrm{s.t.} && \sum_{i\in S^* \oplus S} {1\over N_i} \le \Delta_{S^*, S}^2, \quad\forall S \in \mathcal{I}, S \ne S^*
\end{eqnarray*}

The following result shows the relationships between $H_{0,c}$, $H_{1,c}$ and $H_{2,c}$.
\begin{proposition}\label{Theorem_basic}
For any combinatorial pure exploration instance, $H_{0,c} \le H_{1,c} \le H_{2,c}$.
\end{proposition}
\begin{proof}
Since $\width \ge 1$, we have that $H_{1,c} \le H_{2,c}$. 
As for the first inequality, note that $\forall i \in [m]$, $N_i = {\width \over \Delta_{i,c}^2}$ is a feasible solution of the above convex program. Hence we have that $H_{0,c} \le H_{1,c}$.
\end{proof}

\section{Thompson Sampling-based Pure Exploration Algorithm}

Note that pure exploration of classic multi-armed bandit is a special case of combinatorial multi-armed bandit (i.e., $\mathcal{I} = \{\{1\}, \{2\}, \cdots, \{m\}\}$). 
Therefore, in this section, we mainly focus on the TS-based pure exploration algorithm in the combinatorial multi-armed bandit setting. Its framework and analysis can directly lead to results in the classic multi-armed bandit setting.

In the following, we let $\Phi(x,\mu,\sigma^2) \triangleq \Pr_{X \sim \mathcal{N}(\mu,\sigma^2)}[X \ge x]$. 
For any $x \in (0,0.5)$, we also define $\phi(x)$ as a function of $x$ such that $\Phi(\phi(x),0,1) = x$.
%


\subsection{Algorithm Framework}

\begin{algorithm}[t]
    \centering
    \caption{TS-Explore}\label{Algorithm_TSE}
    \begin{algorithmic}[1]
    \STATE \textbf{Input: } Error constraint $\delta$, $q \in [\delta, 0.1]$, $t \gets m$, $N_i \gets 0, R_i \gets 0$ for all $i\in [m]$.
    \STATE Pull each arm once, update their number of pulls $N_i$'s, the sum of their observations $R_i$'s. 
    \WHILE {\textbf{true}}
    \STATE $t \gets t+1$.
    \STATE For all base arm $i\in [m]$, $\hat{\mu}_i(t) \gets {R_i \over N_i}$.
    \STATE $\bm{\hat{\mu}}(t) \gets [\hat{\mu}_1(t), \hat{\mu}_2(t), \cdots, \hat{\mu}_m(t)]$.
    \STATE $\hat{S}_t \gets {\sf Oracle}(\bm{\hat{\mu}}(t))$.
    \FOR {$k=1,2,\cdots, M(\delta,q,t)$}
    \STATE For each arm $i$, draw $\theta_i^k(t)$ independently from distribution $\mathcal{N}(\hat{\mu}_i(t), {C(\delta,q,t)\over N_i})$.
    \STATE $\bm{\theta}^k(t) \gets [\theta_1^k(t), \theta_2^k(t), \cdots, \theta_m^k(t)]$.
    \STATE $\tilde{S}^k_t \gets {\sf Oracle}(\bm{\theta}^k(t))$.
    \STATE $\tilde{\Delta}^k_t \gets \sum_{i\in \tilde{S}^k_t} \theta_i^k(t) -  \sum_{i\in \hat{S}_t} \theta_i^k(t)$.
    \ENDFOR
    \IF {$\forall 1 \le k \le M(\delta,q,t)$, $\tilde{S}^k_t = \hat{S}_t$}
    \STATE \textbf{Return: } $\hat{S}_t$.
    \ELSE
    \STATE $k^*_t \gets \argmax_k \tilde{\Delta}^k_t$, $\tilde{S}_t \gets \tilde{S}^{k^*_t}_t$.
    \STATE Pull arm $i(t) \gets \argmin_{i\in \hat{S}_t \oplus \tilde{S}_t} N_i$, update its number of pulls $N_{i(t)}$ and the sum of its observations $R_{i(t)}$. 
    \ENDIF
    \ENDWHILE
    \end{algorithmic}
\end{algorithm}

Our Thompson Sampling-based algorithm (TS-Explore) is described in Algorithm \ref{Algorithm_TSE}.
We use $N_i$ to denote the number of observations on arm $i$, $R_i$ to denote the sum of all the observations from arm $i$, and $N_i(t), R_i(t)$ to denote the value of $N_i, R_i$ at the beginning of time step $t$. 

%

At each time step $t$, for any $i\in[m]$, TS-Explore first draws $M(\delta,q,t) \triangleq {1\over q}\log (12|\mathcal{I}|^2t^2/\delta)$ random samples $\{\theta_i^k(t)\}_{k=1}^{M(\delta,q,t)}$ independently from a Gaussian distribution with mean $\hat{\mu}_i(t) \triangleq {R_i(t) \over N_i(t)}$ (i.e., the empirical mean of arm $i$) and variance ${C(\delta,q,t)\over N_i(t)}$ (i.e., inversely proportional to $N_i(t)$), where $C(\delta,q,t) \triangleq {\log (12|\mathcal{I}|^2t^2/\delta) \over \phi^2(q)}$. 
Then it checks which super arm is optimal in the empirical mean parameter set $\bm{\hat{\mu}}(t) = [\hat{\mu}_1(t), \hat{\mu}_2(t), \cdots, \hat{\mu}_m(t)]$ and the $k$-th sample set $\bm{\theta}^k(t) = [\theta_1^k(t), \theta_2^k(t), \cdots, \theta_m^k(t)]$ for all $k$ by using the offline ${\sf Oracle}$,
i.e., $\hat{S}_t = {\sf Oracle}(\bm{\hat{\mu}}(t))$ and $\tilde{S}^k_t = {\sf Oracle}(\bm{\theta}^k(t))$.
%
%
%
%
%
If all the best super arms $\tilde{S}^k_t$'s are the same as $\hat{S}_t$, then TS-Explore outputs that this super arm is the optimal one. 
Otherwise, for all $1 \le k \le M(\delta,q,t)$, it will compute the reward gap between $\tilde{S}^k_t$ and $\hat{S}_t$ under the $k$-th sample set $\bm{\theta}^k(t)$, and focus on $k^*_t$ with the largest reward gap, i.e., TS-Explore will choose to pull the base arm with the least number of observations in $\hat{S}_t \oplus \tilde{S}^{k^*_t}_t$ (in the following, we use $\tilde{S}_t$ to represent $\tilde{S}^{k^*_t}_t$ to simplify notations). 
%
%
%
%
%
Note that $\hat{S}_t \ne \tilde{S}_t$ (otherwise we will output $\hat{S}_t$ as the optimal super arm), thus the rule of choosing arms to pull is well defined. 

\subsection{Analysis of TS-Explore}\label{Section_TSPEMAB}

\begin{theorem}\label{Theorem_Explore}
In the CMAB setting, for $q \in [\delta, 0.1]$, with probability at least $1-\delta$, TS-Explore will output the optimal super arm $S^*$ with complexity upper bounded by $O(H_{1,c}(\log{1\over \delta} + \log(|\mathcal{I}|H_{1,c})){\log{1\over \delta} \over \log{1\over q}} + H_{1,c}{\log^2(|\mathcal{I}|H_{1,c}) \over \log{1\over q}})$.
Specifically, if we choose $q = \delta$, then the complexity upper bound is $O(H_{1,c}\log{1\over \delta} + H_{1,c}\log^2(|\mathcal{I}|H_{1,c}))$.
\end{theorem}

\begin{remark}
When the error constraint $\delta \in (0.1, 1)$, we can still let the parameters $(\delta,q)$ in TS-Explore be $(\delta_0,q_0) = (0.1, 0.1)$.
In this case: i) its error probability is upper bounded by $\delta_0 = 0.1 < \delta$; 
and ii) since ${\delta\over \delta_0} < {1\over 0.1} = 10$, the complexity of TS-Explore is still upper bounded by $O(H_{1,c}\log{1\over \delta_0} + H_{1,c}\log^2(|\mathcal{I}|H_{1,c})) = O(H_{1,c}\log{1\over \delta} + H_{1,c}\log^2(|\mathcal{I}|H_{1,c}))$. 
\end{remark}

%
%

%
%


By Theorem \ref{Theorem_Explore}, we can directly obtain the correctness guarantee and the complexity upper bound for applying TS-Explore in pure exploration of classic multi-armed bandit. 

\begin{corollary}\label{Corollary_Complexity}
In the MAB setting, for $q \in [\delta, 0.1]$, with probability at least $1-\delta$, TS-Explore will output the optimal arm $a^*$ with complexity upper bounded by $O(H_{m}(\log{1\over \delta} + \log(mH_{m})){\log{1\over \delta} \over \log{1\over q}} + H_{m}{\log^2(mH_m) \over \log{1\over q}})$.
Specifically, if we choose $q = \delta$, then the complexity upper bound is $O(H_{m}\log{1\over \delta} + H_{m}\log^2(mH_{m}))$.
\end{corollary}

\begin{remark}
The value $q$ in TS-Explore is used to control the number of times that we draw random samples at each time step. 
Note that $M(\delta,q,t) = {1\over q}\log (12|\mathcal{I}|^2t^2/\delta)$. Hence when $q$ becomes larger, we need fewer samples, but the complexity bound becomes worse. 
Here is a trade-off between the algorithm's complexity and the number of random samples it needs to draw. 
Our analysis shows that using $q = \delta^{1\over \beta}$ for some constant $\beta \ge 1$ can make sure that the complexity upper bound remains the same order, and reduce the number of random samples significantly.
\end{remark}

\begin{remark}
If the value $|\mathcal{I}|$ is unknown (which is common in real applications), we can use $2^m$ instead (in $M(\delta,q,t)$ and $C(\delta,q,t)$).
This only increases the constant term in the complexity upper bound 
and does not influence the major term $O(H_{1,c}\log{1\over \delta})$. 
\end{remark}


%
Theorem \ref{Theorem_Explore} shows that 
the complexity upper bound of TS-Explore is $\width$ lower than the CLUCB policy in \citep{Chen2014Combinatorial}. 
%
%
%
To the best of our knowledge, TS-Explore is the first algorithm that efficiently achieves an $O(H_{1,c}\log{1\over \delta})$ complexity upper bound in the most general combinatorial pure exploration setting. 
Besides, this is also the first theoretical analysis of using a TS-based algorithm (i.e., using random samples to make decisions) to deal with combinatorial pure exploration problems under the frequentist setting. 

Though the complexity bound of TS-Explore still has some gap with the optimal one $O(H_{0,c}\log{1\over \delta})$, we emphasize that this is because we only use the simple offline oracle and do not seek more detailed information about the combinatorial structure. This makes our policy more efficient.
%
%
As a comparison, the existing optimal policies \citep{chen2017nearly,jourdan2021efficient} either suffer from an exponential time cost, or require the combinatorial structure to satisfy some specific constraints so that they can adopt more powerful offline oracles that explore detailed information about the combinatorial structure efficiently (please see Appendix \ref{Appendix_C} for discussions about this). 
Therefore, our algorithm is more efficient in the most general CMAB setting, and can be attractive in real applications with large scale 
and complex combinatorial structures. 

On the other hand, the gap between the complexity upper bound and the complexity lower bound does not exist in the MAB setting. 
Corollary \ref{Corollary_Complexity} shows that the complexity upper bound of applying TS-Explore in the classic MAB setting matches the complexity lower bound $\Omega(H_m\log{1\over \delta})$ \citep{kalyanakrishnan2012PAC} when $\delta \to 0$, i.e., TS-Explore is asymptotically optimal in the classic MAB setting. 

Now we provide the proof of Theorem \ref{Theorem_Explore}.
%
%

In this proof, we denote $\mathcal{J} = \{U: \exists S,S' \in \mathcal{I}, U = S \setminus S'\}$. 
Note that $|\mathcal{J}| \le |\mathcal{I}|^2$ and for any $U \in \mathcal{J}$, $|U| \le \width$. We also denote $L_1(t) = \log (12|\mathcal{I}|^2t^2/\delta)$ and $L_2(t) = \log(12|\mathcal{I}|^2t^2M(\delta,q,t)/\delta)$ to simplify notations.

We first define three events as follows:
$\mathcal{E}_{0,c}$ is the event that for all $t > 0$, $U \in \mathcal{J}$, 
\begin{equation*}
    |\sum_{i\in U}  (\hat{\mu}_i(t) - \mu_i)| \le \sqrt{ \sum_{i\in U}{1\over 2N_i(t)}L_1(t)};
\end{equation*}
$\mathcal{E}_{1,c}$ is the event that for all $t > 0$, $1 \le k \le M(\delta,q,t)$, $U \in \mathcal{J}$, 
\begin{eqnarray*}
    |\sum_{i\in U}  (\theta_i^k(t) - \hat{\mu}_i(t))| \le \sqrt{\sum_{i\in U}{2C(\delta,q,t)\over N_i(t)}L_2(t)};
\end{eqnarray*}
and $\mathcal{E}_{2,c}$ is the event that for all $t > 0$, $S,S' \in \mathcal{I}$. there exists $1 \le k \le M(\delta,q,t)$ such that 
\begin{equation*}
    \sum_{i\in S} \theta_i^k(t) - \sum_{i\in S'} \theta_i^k(t) \ge \sum_{i\in S} \mu_i - \sum_{i\in S'} \mu_i.
\end{equation*}

Roughly speaking, $\mathcal{E}_{0,c}$ means that for any $U \in \mathcal{J}$, the gap between its real mean and its empirical mean lies in the corresponding confidence radius; $\mathcal{E}_{1,c}$ means that for any $U \in \mathcal{J}$, the gap between its empirical mean and the sum of its random samples lies in the corresponding confidence radius; and $\mathcal{E}_{2,c}$ means that for any two super arms $S\ne S'$, at each time step, there is at least one sample set such that the gap between the sum of their random samples is larger than the gap between their real means. 

%


We will first prove that $\mathcal{E}_{0,c} \land \mathcal{E}_{1,c} \land \mathcal{E}_{2,c}$ happens with high probability, and then show the correctness guarantee and complexity upper bound under $\mathcal{E}_{0,c} \land \mathcal{E}_{1,c} \land \mathcal{E}_{2,c}$.

\begin{lemma}\label{Lemma_123}
In Algorithm \ref{Algorithm_TSE}, we have that
\begin{equation*}
    \Pr[\mathcal{E}_{0,c} \land \mathcal{E}_{1,c} \land \mathcal{E}_{2,c}] \ge 1 - \delta.
\end{equation*}
\end{lemma}

%
\begin{proof} 
Note that the random variable $(\hat{\mu}_i(t) - \mu_i)$ is zero-mean and ${1\over 4N_i(t)}$ sub-Gaussian, and for different $i$, the random variables $(\hat{\mu}_i(t) - \mu_i)$'s are independent. Therefore, $\sum_{i\in U}  (\hat{\mu}_i(t) - \mu_i)$ is zero-mean and $\sum_{i\in U}{1\over 4N_i(t)}$ sub-Gaussian. Then by concentration inequality of sub-Gaussian random variables (see details in Appendix \ref{Appendix_D}), 
\begin{eqnarray*}
    &&\Pr\left[|\sum_{i\in U}  (\hat{\mu}_i(t) - \mu_i)| > \sqrt{\sum_{i\in U}{1\over 2N_i(t)}L_1(t)}\right] \\
    &\le& 2\exp(-L_1(t))\\
    &=& {\delta \over 6|\mathcal{I}|^2t^2}.
\end{eqnarray*}

This implies that
\begin{equation*}
    \Pr[\neg \mathcal{E}_{0,c}] \le \sum_{U,t} {\delta \over 6|\mathcal{I}|^2t^2} \le \sum_{t} {\delta \over 6t^2} \le {\delta \over 3},
\end{equation*}
where the second inequality is because that $|\mathcal{J}| \le |\mathcal{I}|^2$, and the third inequality is because that $\sum_t {1\over t^2} \le 2$.

Similarly, the random variable $(\theta_i^k(t) - \hat{\mu}_i(t))$ is a zero-mean Gaussian random variable with variance ${C(\delta, q, t)\over N_i(t)}$, and for different $i$, the random variables $(\theta_i^k(t) - \hat{\mu}_i(t))$'s are also independent.
Then by concentration inequality, 
\begin{eqnarray*}
    \!\!\!\!\!\!\!&&\Pr\left[|\sum_{i\in U}  (\theta_i^k(t) - \hat{\mu}_i(t))| > \sqrt{\sum_{i\in U}{2C(\delta,q,t)\over N_i(t)}L_2(t)}\right]\\
    \!\!\!\!\!\!\!&\le& 2\exp(-L_2(t)) \\
    \!\!\!\!\!\!\!&=& {\delta \over 6|\mathcal{I}|^2t^2M(\delta,q,t)}.
\end{eqnarray*}

This implies that
\begin{equation*}
    \Pr[\neg \mathcal{E}_{1,c}] \le \sum_{U,t, k} {\delta \over 6|\mathcal{I}|^2t^2M(\delta,q,t)} \le \sum_{U,t} {\delta \over 6|\mathcal{I}|^2t^2} \le {\delta \over 3},
\end{equation*}
where the second inequality is because that there are totally $M(\delta,q,t)$ sample sets at time step $t$.

Finally we consider the probability $\Pr[\neg \mathcal{E}_{2,c} | \mathcal{E}_{0,c}]$. In the following, we denote $\Delta_{S,S'} \triangleq \sum_{i\in S} \mu_i - \sum_{i\in S'} \mu_i = \sum_{i\in S\setminus S'} \mu_i - \sum_{i\in S'\setminus S} \mu_i $ as the reward gap (under the real means) between $S$ and $S'$. 

For any fixed $S \ne S'$, we denote $A(t) = \sum_{i\in S \setminus S'} {1\over 2N_i(t)}$, $B(t) = \sum_{i\in S' \setminus S} {1\over 2N_i(t)}$ and $C(t) = C(\delta, q,t)$ to simplify notations.
Then under event $\mathcal{E}_{0,c}$, we must have that $\sum_{i\in S \setminus S'} \hat{\mu}_{i}(t) - \sum_{i\in S'\setminus S } \hat{\mu}_{i}(t) \ge (\sum_{i\in  S \setminus S'} \mu_{i} - \sqrt{A(t)L_1(t)}) - (\sum_{i\in S'\setminus S } \mu_{i} +  \sqrt{B(t)L_1(t)}) \ge \Delta_{S,S'} - \sqrt{A(t)L_1(t)} - \sqrt{B(t)L_1(t)}$. 
%
%

Since $\sum_{i\in S\setminus S' } \theta_{i}^k(t) - \sum_{i\in S'\setminus S } \theta_{i}^k(t)$ is a Gaussian random variable with mean $\sum_{i\in S \setminus S'} \hat{\mu}_{i}(t) - \sum_{i\in S'\setminus S } \hat{\mu}_{i}(t)$ and variance $2A(t)C(t) + 2B(t)C(t)$, then under event $\mathcal{E}_{0,c}$, 
(recall that $\Phi(x,\mu,\sigma^2) = \Pr_{X \sim \mathcal{N}(\mu, \sigma^2)}[X \ge x]$):
\begin{eqnarray}
\!\!\!\!\!\nonumber&&\!\!\!\!\!\Pr\left[\sum_{i\in S} \theta_i^k(t) - \sum_{i\in S'} \theta_i^k(t) \ge \Delta_{S,S'}\right] \\
\nonumber\!\!\!\!\!&=&\!\!\!\!\! \Pr\left[\sum_{i\in S \setminus S'} \theta_i^k(t) - \sum_{i\in S' \setminus S} \theta_i^k(t) \ge \Delta_{S,S'}\right]\\
\!\!\!\!\!\nonumber&=&\!\!\!\!\! \Phi\!\left(\!\Delta_{S,S'}, \!\!\!\sum_{i\in S \setminus S'} \!\!\hat{\mu}_{i}(t) - \!\!\!\!\!\sum_{i\in S'\setminus S } \!\!\hat{\mu}_{i}(t), 2A(t)C(t) + 2B(t)C(t)\!\right)\\
\!\!\!\!\!\nonumber&\ge&\!\!\!\!\! \Phi(\Delta_{S,S'}, \Delta_{S,S'} - \sqrt{A(t)L_1(t)} - \sqrt{B(t)L_1(t)}, \\
\!\!\!\!\!\nonumber&&\quad \quad\quad\quad\quad\quad\quad\quad\quad\quad\quad\quad\quad 2A(t)C(t) + \!2B(t)C(t))\\
\!\!\!\!\!\nonumber&=&\!\!\!\!\! \Phi(\sqrt{A(t)L_1(t)} + \!\sqrt{B(t)L_1(t)}, 0,  2A(t)C(t) + \!2B(t)C(t))\\
\!\!\!\!\!\nonumber&=&\!\!\!\!\! \Phi\left(\sqrt{L_1(t)\over C(t)} \cdot {\sqrt{A(t)} + \sqrt{B(t)} \over \sqrt{2A(t)+2B(t)}}, 0,  1\right)\\
\!\!\!\!\!\nonumber&\ge&\!\!\!\!\!\Phi\left(\sqrt{L_1(t)\over C(t)}, 0, 1\right)\\
\!\!\!\!\!\nonumber&=&\!\!\!\!\! q,
\end{eqnarray}
where the last equation is because that we choose $C(t) = {L_1(t) \over \phi^2(q)}$ and $\Phi(\phi(q), 0,1) = q$ (by definition of $\phi$). 

Note that the parameter sets $\{\bm{\theta}^k(t)\}_{k=1}^{M(\delta,q,t)}$ are chosen independently, therefore under event $\mathcal{E}_{0,c}$, we have that
\begin{eqnarray*}
&&\Pr\left[\forall k, \sum_{i\in S} \theta_i^k(t) - \sum_{i\in S'} \theta_i^k(t) < \Delta_{S,S'}\right] \\
&\le& (1-q)^{M(\delta,q,t)} \\
&\le& {\delta \over 12|\mathcal{I}|^2t^2},
\end{eqnarray*}
where that last inequality is because that we choose $M(\delta,q,t) = {1\over q}\log (12|\mathcal{I}|^2t^2/\delta)$.

This implies that
\begin{equation*}
    \Pr[\neg \mathcal{E}_{2,c} | \mathcal{E}_{0,c}] \le \sum_{t, S, S'} {\delta \over 12|\mathcal{I}|^2t^2} \le \sum_{t} {\delta \over 12t^2} \le {\delta \over 3}.
\end{equation*}

All these show that $\Pr[\mathcal{E}_{0,c} \land \mathcal{E}_{1,c} \land \mathcal{E}_{2,c}] \ge 1 - \delta$.
\end{proof}

%

%
Then it is sufficient to prove that under event $\mathcal{E}_{0,c} \land \mathcal{E}_{1,c} \land \mathcal{E}_{2,c}$, TS-Explore works correctly with complexity upper bound shown in Theorem \ref{Theorem_Explore}.


Firstly, we prove that TS-Explore will output $S^*$. The proof is quite straightforward: if $\hat{S}_t \ne S^*$, then under event $\mathcal{E}_{2,c}$, there exists $k$ such that $\sum_{i\in S^*} \theta^k_i(t) - \sum_{i\in \hat{S}_t} \theta^k_i(t) \ge \Delta_{S^*,\hat{S}_t} > 0$. Therefore $\tilde{S}^k_t \ne \hat{S}_t$ and we will not output $\hat{S}_t$. Because of this, TS-Explore can only return $S^*$ under event $\mathcal{E}_{0,c} \land \mathcal{E}_{1,c} \land \mathcal{E}_{2,c}$, and we finish the proof of its correctness.

Then we come to bound the complexity of TS-Explore, and we will use the following lemma 
in our analysis. 

\begin{restatable}{lemma}{LemmaComplexity}\label{Lemma_Complexity}
Under event $\mathcal{E}_{0,c} \land \mathcal{E}_{1,c} \land \mathcal{E}_{2,c}$, a base arm $i$ will not be pulled if $N_i(t) \ge {98\width C(t)L_2(t) \over \Delta_{i,c}^2}$.
\end{restatable}

Due to space limit, we only prove Lemma \ref{Lemma_Complexity} for the case that $(\hat{S}_t = S^*) \lor (\tilde{S}_t = S^*)$ 
in our main text, and defer the complete proof to Appendix \ref{Appendix_A}. 

\begin{proof}
We will prove this lemma by contradiction.

First we consider the case that $(\hat{S}_t = S^* )\land (\tilde{S}_t \ne S^*)$. In this case, $i\in S^* \oplus \tilde{S}_t$, which implies that $\Delta_{i,c} \le \Delta_{S^*, \tilde{S}_t}$. If we choose a base arm $i$ with $N_i(t) \ge {98\width C(t)L_2(t) \over \Delta_{i,c}^2} \ge {98\width C(t)L_2(t) \over \Delta_{S^*, \tilde{S}_t}^2}$ to pull, we know that $\forall j \in S^* \oplus \tilde{S}_t$, $N_j(t) \ge {98\width C(t)L_2(t) \over \Delta_{S^*, \tilde{S}_t}^2}$. This implies that $\sqrt{\sum_{j\in S^* \setminus \tilde{S}_t} {2\over N_j(t)}} \le {\Delta_{S^*, \tilde{S}_t}\over 7\sqrt{C(t)L_2(t)}}$ and $\sqrt{\sum_{j\in \tilde{S}_t \setminus S^*} {2\over N_j(t)}} \le {\Delta_{S^*, \tilde{S}_t}\over 7\sqrt{C(t)L_2(t)}}$.

By $\mathcal{E}_{2,c}$, there exists $k$ such that $\sum_{j\in S^* \setminus \tilde{S}_t} \theta_j^k(t) - \sum_{j\in \tilde{S}_t \setminus S^*} \theta_j^k(t) \ge \Delta_{S^*, \tilde{S}_t}$. By $\mathcal{E}_{1,c}$, $|\sum_{j\in S^* \setminus \tilde{S}_t} (\theta_j^k(t) - \hat{\mu}_j(t))| \le \sqrt{\sum_{j\in S^* \setminus \tilde{S}_t}{2C(t)\over N_j(t)}L_2(t)} \le {\Delta_{S^*, \tilde{S}_t}\over 7}$ and similarly $|\sum_{j\in \tilde{S}_t \setminus S^*} (\theta_j^k(t) - \hat{\mu}_j(t))| \le {\Delta_{S^*, \tilde{S}_t}\over 7}$. Hence
\begin{eqnarray*}
\!\!\!\!\!&&\sum_{j\in S^* \setminus \tilde{S}_t} \hat{\mu}_j(t) - \sum_{j\in \tilde{S}_t \setminus S^*} \hat{\mu}_j(t) \\
\!\!\!\!\!&\ge&  \sum_{j\in S^* \setminus \tilde{S}_t} \theta_j^k(t) - \sum_{j\in \tilde{S}_t \setminus S^*} \theta_j^k(t)\\
\!\!\!\!\!&&- |\!\!\!\!\!\sum_{j\in S^* \setminus \tilde{S}_t} (\theta_j^k(t) - \hat{\mu}_j(t))| - |\!\!\!\!\!\sum_{j\in \tilde{S}_t \setminus S^*} (\theta_j^k(t) - \hat{\mu}_j(t))|\\
\!\!\!\!\!&\ge&{5\Delta_{S^*, \tilde{S}_t}\over 7}.
\end{eqnarray*}

$\mathcal{E}_{1,c}$ also means $|\sum_{j\in S^* \setminus \tilde{S}_t} (\theta_j^{k_t^*}(t) - \hat{\mu}_j(t))| \le {\Delta_{S^*, \tilde{S}_t}\over 7}$ and $|\sum_{j\in \tilde{S}_t \setminus S^*} (\theta_j^{k_t^*}(t) - \hat{\mu}_j(t))| \le {\Delta_{S^*, \tilde{S}_t}\over 7}$. Thus we know that 
\begin{eqnarray*}
\!\!\!\!\!&&\sum_{j\in S^* \setminus \tilde{S}_t} \theta_j^{k_t^*}(t) - \sum_{j\in \tilde{S}_t \setminus S^*} \theta_j^{k_t^*}(t) \\
\!\!\!\!\!&\ge&  \sum_{j\in S^* \setminus \tilde{S}_t} \hat{\mu}_j(t) - \sum_{j\in \tilde{S}_t \setminus S^*} \hat{\mu}_j(t)\\
\!\!\!\!\!&&- |\!\!\!\!\!\sum_{j\in S^* \setminus \tilde{S}_t} (\theta_j^{k_t^*}(t) - \hat{\mu}_j(t))| - |\!\!\!\!\!\sum_{j\in \tilde{S}_t \setminus S^*} (\theta_j^{k_t^*}(t) - \hat{\mu}_j(t))|\\
\!\!\!\!\!&\ge&{3\Delta_{S^*, \tilde{S}_t}\over 7}.
\end{eqnarray*}

This contradicts with the fact that $\tilde{S}_t$ is the optimal super arm under the $k_t^*$-th sample set $\bm{\theta}^{k_t^*}(t)$.

Then we come to the case that $(\hat{S}_t \ne S^*) \land (\tilde{S}_t = S^*)$. In this case, $i\in S^* \oplus \hat{S}_t$, which implies that $\Delta_{i,c} \le \Delta_{S^*, \hat{S}_t}$. If we choose a base arm $i$ with $N_i(t) \ge {98\width C(t)L_2(t) \over \Delta_{i,c}^2} \ge {98\width C(t)L_2(t) \over \Delta_{S^*, \hat{S}_t}^2}$ to pull, we have that $\sqrt{\sum_{j\in S^* \setminus \hat{S}_t} {2\over N_j(t)}} \le {\Delta_{S^*, \hat{S}_t}\over 7\sqrt{C(t)L_2(t)}}$ and $\sqrt{\sum_{j\in \hat{S}_t \setminus S^*} {2\over N_j(t)}} \le {\Delta_{S^*, \hat{S}_t}\over 7\sqrt{C(t)L_2(t)}}$.

By $\mathcal{E}_{2,c}$, there exists $k$ such that $\sum_{j\in S^* \setminus \hat{S}_t} \theta_j^k(t) - \sum_{j\in \hat{S}_t \setminus S^*} \theta_j^k(t) \ge \Delta_{S^*, \hat{S}_t}$. 
By $\mathcal{E}_{1,c}$, $|\sum_{j\in S^* \setminus \hat{S}_t} (\theta_j^k(t) - \hat{\mu}_j(t))| \le {\Delta_{S^*, \hat{S}_t}\over 7}$ and $|\sum_{j\in \hat{S}_t \setminus S^*} (\theta_j^k(t) - \hat{\mu}_j(t))| \le {\Delta_{S^*, \hat{S}_t}\over 7}$. All these imply 
\begin{eqnarray*}
\!\!\!\!\!&&\sum_{j\in S^* \setminus \hat{S}_t} \hat{\mu}_j(t) - \sum_{j\in \hat{S}_t \setminus S^*} \hat{\mu}_j(t) \\
\!\!\!\!\!&\ge&  \sum_{j\in S^* \setminus \hat{S}_t} \theta_j^k(t) - \sum_{j\in \hat{S}_t \setminus S^*} \theta_j^k(t)\\
\!\!\!\!\!&&- |\!\!\!\!\!\sum_{j\in S^* \setminus \hat{S}_t} (\theta_j^k(t) - \hat{\mu}_j(t))| - |\!\!\!\!\!\sum_{j\in \hat{S}_t \setminus S^*} (\theta_j^k(t) - \hat{\mu}_j(t))|\\
\!\!\!\!\!&\ge&{5\Delta_{S^*, \hat{S}_t}\over 7}.
\end{eqnarray*}

This contradicts with the fact that $\hat{S}_t$ is the empirically optimal super arm.
\end{proof}

Lemma \ref{Lemma_Complexity} is similar to Lemma 10 in \cite{Chen2014Combinatorial}. Both of them give an upper bound for the number of pulls on base arm $i$. The key difference is that in \cite{Chen2014Combinatorial}, for an arm set $U \in \mathcal{J}$, the gap between its real mean and its upper confidence bound is $\tilde{O}(\sum_{i\in U} \sqrt{1\over N_i})$, which means that we require all the $N_i$'s to be $\tilde{\Theta}({\width^2 
\over\Delta_{i,c}^2})$ to make sure that this gap is less than $\Delta_{i,c}$. In our paper, based on event $\mathcal{E}_{0,c} \land \mathcal{E}_{1,c}$, the gap between $U$'s real mean and the sum of random samples in $U$ is $\tilde{O}(\sqrt{\sum_{i\in U}{C(t) \over  N_i}})$. Therefore, we only require all the $N_i$'s to be $\tilde{\Theta}({\width C(t) \over \Delta_{i,c}^2})$ to make sure that this gap is less than $\Delta_{i,c}$, and this reduces a factor of $\width$ in the number of pulls on base arm $i$ (our analysis shows that $C(t)$ is approximately a constant).
In fact, reducing a $\width$ factor in Lemma \ref{Lemma_Complexity} is the key reason that the complexity upper bound of TS-Explore is $\width$ lower than the CLUCB policy in \citep{Chen2014Combinatorial}. 

The novelty of our proof for Lemma \ref{Lemma_Complexity} mainly lies in the event $\mathcal{E}_{2,c}$ (as well as the mechanism of recording all the $\tilde{\Delta}^k_t$'s and focusing on the largest one), i.e., we show that under event $\mathcal{E}_{2,c}$, $\tilde{S}_t$ has some similar properties as the super arm with the largest upper confidence bound (which is used in LUCB-based policies such as CLUCB).
For example, when $\hat{S}_t$ does not equal the optimal super arm $S^*$, $\mathcal{E}_{2,c}$ tells us that there must exist $k$ such that $\sum_{i\in S^*} \theta_i^k(t) - \sum_{i\in \hat{S}_t} \theta_i^k(t) \ge \Delta_{S^*, \hat{S}_t}$.
Along with the fact $k^*_t = \argmax_k \tilde{\Delta}^k_t$, we know $\tilde{\Delta}^{k^*_t }_t \ge \Delta_{S^*, \hat{S}_t}$. 
This means that the reward gap between $\hat{S}_t$ and $\tilde{S}_t$ ($\tilde{S}^{k^*_t }_t$) could be larger than $\Delta_{S^*, \hat{S}_t}$, which implies that $\tilde{S}_t$ is either an insufficiently learned sub-optimal arm or the optimal arm (see details in the complete proof in Appendix \ref{Appendix_A}). 
This method solves the challenge of the uncertainty in random samples, and allows us to use similar analysis techniques (e.g., \cite{Chen2014Combinatorial}) to prove Lemma \ref{Lemma_Complexity}.


By Lemma \ref{Lemma_Complexity}, if $\forall i, N_i(t) \ge {98\width C(t)L_2(t) \over \Delta_{i,c}^2}$, then TS-Explore must terminate (and output the correct answer). 
Thus, the complexity $Z$ satisfies 
\begin{equation*}
    Z \le \sum_{i\in [m]} {98\width C(Z)L_2(Z) \over\Delta_{i,c}^2} = 98H_{1,c} C(Z)L_2(Z).
\end{equation*}

For $q \le 0.1$, ${1\over \phi^2(q)} = O({1\over \log{1\over q}})$. Then with $C(Z) = {\log (12|\mathcal{I}|^2Z^2/\delta) \over \phi^2(q)}$, $L_2(Z) = \log(12|\mathcal{I}|^2Z^2M(Z)/\delta)$ and $M(Z) = {\log(12|\mathcal{I}|^2Z^2/\delta)\over q}$, we have that (note that $q \ge \delta$):
\begin{equation}\label{Eq_1}
    Z \le O\left(H_{1,c}{\left(\log (|\mathcal{I}|Z) + \log{1\over \delta}\right)^2 \over \log{1\over q}}\right).
\end{equation}

Therefore, after some basic calculations (the details are deferred to Appendix \ref{Appendix_B}), we know that
\begin{equation*}
 Z = O\left(H_{1,c}{\left(\log (|\mathcal{I}|H_{1,c}) + \log{1\over \delta}\right)^2 \over \log{1\over q}}\right).
\end{equation*}

\section{Experiments}\label{Appendix_Experiments}

%
%

In this section, we conduct some experiments to compare the complexity performances of efficient learning algorithms (i.e., TS-Explore and CLUCB \cite{Chen2014Combinatorial}) and the optimal but non-efficient algorithm NaiveGapElim \cite{chen2017nearly} in the CMAB setting.
%
%
We consider the following combinatorial pure exploration problem instance with parameter $n \in \mathbb{N}^+$, and always choose $q = \delta$ in TS-Explore. 
All the results (average complexities and their standard deviations) take an average of 100 independent runs.

\begin{problem}
For fixed value $n$, there are totally $2n$ base arms. For the first $n$ base arms, their expected rewards equal 0.1, and for the last $n$ base arms, their expected rewards equal 0.9. There are only two super arms: 
$S_1$ contains the first $n$ base arms and $S_2$ contains the last $n$ base arms. 
\end{problem}


We first fix $\delta = 10^{-3}$, and compare the complexity of the above algorithms under different $n$'s (Fig. \ref{Figure_1}). We can see that when $n$ increases, the complexities of TS-Explore and NaiveGapElim do not increase a lot, while the complexity of CLUCB increases linearly. This accords with our analysis, since $H_{0,c}(n) = H_{1,c}(n) = 2n \cdot {2n \over (0.8n)^2} = 6.25$ is a constant but $H_{2,c}(n) = 2n \cdot {(2n)^2 \over (0.8n)^2} = 12.5n$ is linear with $n$ (here $H_{0,c}(n), H_{1,c}(n), H_{2,c}(n)$ are the values of $H_{0,c}, H_{1,c}, H_{2,c}$ under the problem instance with parameter $n$, respectively). 


Then we fix $n = 2$, and compare the complexity of the above algorithms under different $\delta$'s (Fig. \ref{Figure_2}). 
We can see that when $\delta$ is large, the complexity of TS-Explore decreases as $\delta$ decreases, and when $\delta$ is small, the complexity of TS-Explore increases as $\delta$ decreases. Moreover, the complexities of TS-Explore and NaiveGapElim increase much slower than CLUCB (when $\delta$ decreases). This also accords with our analysis. Note that there is a term $O(H_{1,c}{\log^2 (|\mathcal{I}|H_{1,c}) \over \log{1\over q}})$ in our complexity bound. 
Since we choose $q = \delta$, this term decreases as $\delta$ decreases. When $\delta = 10^{-1}$, this term is very large and becomes the majority term in complexity, and therefore the complexity decreases when $\delta$ decreases from $10^{-1}$ to $10^{-3}$. When $\delta = 10^{-3}$, the term $O(H_{1,c} \log{1\over \delta})$ becomes the majority term in complexity, therefore the complexity increases when $\delta$ decreases from $10^{-3}$ to $10^{-5}$. 

\begin{figure}[t]
\centering 
\subfigure[$\delta = 10^{-3}$]{ \label{Figure_1} 
\includegraphics[width=2.9in]{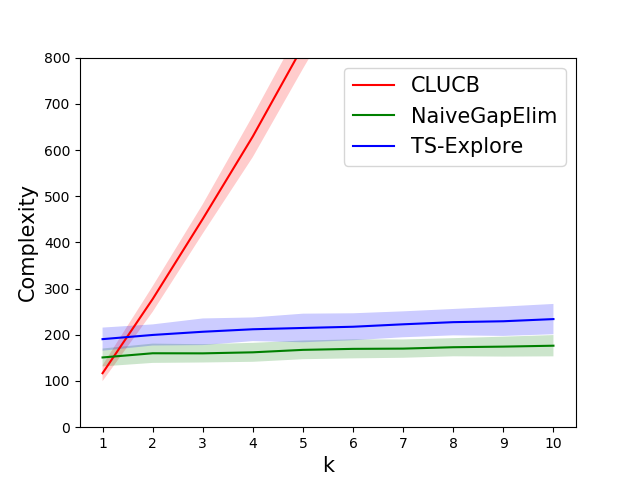}}
\subfigure[$n = 2$]{ \label{Figure_2} 
\includegraphics[width=2.9in]{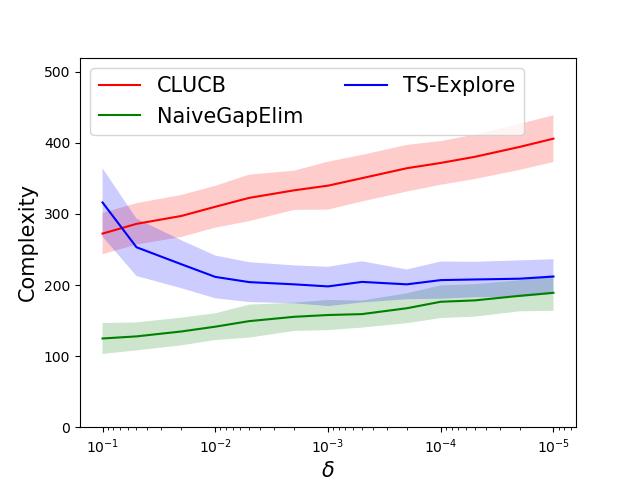}}
\caption{Comparison of TS-Explore, CLUCB and NaiveGapElim.
}\label{Figure_E1}
\end{figure}


All the experimental results indicate that TS-Explore outperforms CLUCB (especially when the size of the problem is large or the error constraint $\delta$ is small). 
On the other hand, the complexity of our efficient algorithm TS-Explore is only a little higher than the optimal but non-efficient algorithm NaiveGapElim. 
%
These results demonstrate the effectiveness of our algorithm.

\section{Conclusions}

In this paper, we explore the idea of Thompson Sampling to solve the pure exploration problems under the frequentist setting.
We first propose TS-Explore, an efficient policy that uses random samples to make decisions, and then show that: i) in the combinatorial multi-armed bandit setting, our policy can achieve a lower complexity bound than existing efficient policies; and ii) in the classic multi-armed bandit setting, our policy can achieve an asymptotically optimal complexity bound. 



There remain many interesting topics to be further studied, e.g., how to achieve the optimal complexity bound in the combinatorial multi-armed bandit setting by seeking detailed information about the combinatorial structure in TS-Explore; and what are the complexity bounds of using our TS-Explore framework in other pure exploration problems (e.g., dueling bandit and linear bandit). It is also worth studying how to design TS-based pure exploration algorithms for the fixed budget setting.

\section*{Acknowledgement}

This work was supported by the National Key Research and Development Program of China
(2017YFA0700904, 2020AAA0104304),
NSFC Projects (Nos. 62061136001, 62106122, 62076147, U19B2034,
U1811461, U19A2081), Beijing NSF Project (No. JQ19016), Tsinghua-Huawei Joint Research Program, and a grant from
Tsinghua Institute for Guo Qiang.


%


\bibliography{example_paper}
\bibliographystyle{icml2022}

\newpage
\appendix
\onecolumn

\section{The Complete Proof of Lemma \ref{Lemma_Complexity}}\label{Appendix_A}

\textbf{Lemma 4.7} (restated) Under event $\mathcal{E}_{0,c} \land \mathcal{E}_{1,c} \land \mathcal{E}_{2,c}$, a base arm $i$ will not be pulled if $N_i(t) \ge {98\width C(t)L_2(t) \over \Delta_{i,c}^2}$. 

\begin{proof}
We still prove this lemma by contradiction.

Assume that arm $i$ is pulled with $N_i(t) \ge {98\width C(t)L_2(t) \over \Delta_{i,c}^2}$, then there are four probabilities: $(i\in \hat{S}_t) \land (i\in S^*)$, $(i\in \hat{S}_t) \land (i\notin S^*)$, $(i\in \tilde{S}_t) \land (i\in S^*)$, $(i\in \tilde{S}_t) \land (i\notin S^*)$.

Case i): $(i\in \hat{S}_t) \land (i\in S^*)$ or $(i\in \tilde{S}_t) \land (i\notin S^*)$. In this case, $i \in S^* \oplus \tilde{S}_t$ and therefore $\Delta_{i,c} \le \Delta_{S^*,\tilde{S}_t}$.

Since $\hat{S}_t = {\sf Oracle}(\bm{\hat{\mu}}(t))$, we have that $\sum_{j\in \tilde{S}_t \setminus \hat{S}_t}  \hat{\mu}_j(t) \le \sum_{j\in \hat{S}_t \setminus \tilde{S}_t}  \hat{\mu}_j(t)$. By event $\mathcal{E}_{1,c}$, we also have that for any $k$, $|\sum_{j\in \tilde{S}_t \setminus \hat{S}_t} (\theta_j^{k}(t) - \hat{\mu}_j(t))| \le \sqrt{C(t)\sum_{j\in \tilde{S}_t \setminus \hat{S}_t}{2\over N_j(t)}L_2(t)} \le {\Delta_{i,c} \over 7}$, and similarly $|\sum_{j\in \hat{S}_t \setminus \tilde{S}_t} (\theta_j^{k}(t) - \hat{\mu}_j(t))| \le \sqrt{C(t)\sum_{j\in \hat{S}_t \setminus \tilde{S}_t}{2\over N_j(t)}L_2(t)} \le {\Delta_{i,c} \over 7}$ (recall that $N_i(t) \ge {98\width C(t)L_2(t) \over \Delta_{i,c}^2}$ and $i = \argmin_{j\in \hat{S}_t \oplus \tilde{S}_t} N_j(t)$).

Therefore, for any $k$, we have that \begin{eqnarray*}
&&\sum_{j\in \tilde{S}_t \setminus \hat{S}_t}  \theta_j^k(t) - \sum_{j\in \hat{S}_t \setminus \tilde{S}_t}  \theta_j^k(t) \\
&=& \left(\sum_{j\in \tilde{S}_t \setminus \hat{S}_t} \hat{\mu}_j(t) - \sum_{j\in \hat{S}_t \setminus \tilde{S}_t} \hat{\mu}_j(t)\right) + \left(\sum_{j\in \tilde{S}_t \setminus \hat{S}_t} (\theta_j^k(t) - \hat{\mu}_j(t)) \right) - \left( \sum_{j\in \hat{S}_t \setminus \tilde{S}_t} (\theta_j^{k}(t) - \hat{\mu}_j(t)) \right)\\
&\le& \left(\sum_{j\in \tilde{S}_t \setminus \hat{S}_t} \hat{\mu}_j(t) - \sum_{j\in \hat{S}_t \setminus \tilde{S}_t} \hat{\mu}_j(t)\right) + \left(|\sum_{j\in \tilde{S}_t \setminus \hat{S}_t} (\theta_j^k(t) - \hat{\mu}_j(t)) |\right) + \left(| \sum_{j\in \hat{S}_t \setminus \tilde{S}_t} (\theta_j^{k}(t) - \hat{\mu}_j(t)) |\right)\\
&\le& 0  +  {2\over 7}\Delta_{i,c} \\
&\le& {2\over 7}\Delta_{i,c}.
\end{eqnarray*}
This means that $\tilde{\Delta}_t^{k_t^*}  = \sum_{j\in \tilde{S}_t \setminus \hat{S}_t}  \theta_j^{k_t^*}(t) - \sum_{j\in \hat{S}_t \setminus \tilde{S}_t}  \theta_j^{k_t^*}(t) \le {2\over 7}\Delta_{i,c}$.


Moreover, since $\sum_{j\in \tilde{S}_t \setminus \hat{S}_t} \theta_j^{k_t^*}(t) \ge \sum_{j\in \hat{S}_t \setminus \tilde{S}_t} \theta_j^{k_t^*}(t)$, we have that 
\begin{eqnarray}
\nonumber&&\sum_{j\in \tilde{S}_t \setminus \hat{S}_t}  \hat{\mu}_j(t) - \sum_{j\in \hat{S}_t \setminus \tilde{S}_t}  \hat{\mu}_j(t) \\
\nonumber&=& \left(\sum_{j\in \tilde{S}_t \setminus \hat{S}_t} \theta_j^{k_t^*}(t) - \sum_{j\in \hat{S}_t \setminus \tilde{S}_t} \theta_j^{k_t^*}(t)\right) - \left(\sum_{j\in \tilde{S}_t \setminus \hat{S}_t} (\theta_j^{k_t^*}(t) - \hat{\mu}_j(t)) \right) + \left( \sum_{j\in \hat{S}_t \setminus \tilde{S}_t} (\theta_j^{k_t^*}(t) - \hat{\mu}_j(t)) \right)\\
\nonumber&\ge& \left(\sum_{j\in \tilde{S}_t \setminus \hat{S}_t} \theta_j^{k_t^*}(t) - \sum_{j\in \hat{S}_t \setminus \tilde{S}_t} \theta_j^{k_t^*}(t)\right) - \left(|\sum_{j\in \tilde{S}_t \setminus \hat{S}_t} (\theta_j^{k_t^*}(t) - \hat{\mu}_j(t))| \right) - \left( |\sum_{j\in \hat{S}_t \setminus \tilde{S}_t} (\theta_j^{k_t^*}(t) - \hat{\mu}_j(t))| \right)\\
\nonumber&\ge& 0  - {2\over 7}\Delta_{i,c} \\
\label{Eq_2}&\ge& - {2\over 7}\Delta_{i,c}.
\end{eqnarray}


On the other hand, by event $\mathcal{E}_{2,c}$, we know that $\exists k'$ such that $\sum_{j\in S^*} \theta_j^{k'}(t) - \sum_{j\in \tilde{S}_t} \theta_j^{k'}(t) \ge \Delta_{S^*,\tilde{S}_t}$. 
Then 
\begin{eqnarray}
\nonumber&&\sum_{j\in S^*} \theta_j^{k'}(t) - \sum_{j\in \hat{S}_t} \theta_j^{k'}(t) \\
\nonumber&=&\left(\sum_{j\in S^*} \theta_j^{k'}(t) - \sum_{j\in \tilde{S}_t} \theta_j^{k'}(t) \right) + \left(\sum_{j\in \tilde{S}_t} \theta_j^{k'}(t)  - \sum_{j\in \hat{S}_t} \theta_j^{k'}(t)  \right)\\
\nonumber&\ge& \Delta_{S^*,\tilde{S}_t} + \left(\sum_{j\in \tilde{S}_t} \theta_j^{k'}(t)  - \sum_{j\in \hat{S}_t} \theta_j^{k'}(t)  \right)\\
\nonumber&=& \Delta_{S^*,\tilde{S}_t} + \left(\sum_{j\in \tilde{S}_t \setminus \hat{S}_t} \theta_j^{k'}(t)  - \sum_{j\in \hat{S}_t \setminus \tilde{S}_t} \theta_j^{k'}(t)  \right)\\
\nonumber&=& \Delta_{S^*,\tilde{S}_t} + \left(\sum_{j\in \tilde{S}_t \setminus \hat{S}_t} \hat{\mu}_j(t)  - \sum_{j\in \hat{S}_t \setminus \tilde{S}_t} \hat{\mu}_j(t)  \right) + \left(\sum_{j\in \tilde{S}_t \setminus \hat{S}_t} (\theta_j^{k'}(t)  -  \hat{\mu}_j(t))  \right) - \left(\sum_{j\in \hat{S}_t \setminus \tilde{S}_t} (\theta_j^{k'}(t) -  \hat{\mu}_j(t))   \right)\\
\label{Eq_3}&\ge & \Delta_{S^*,\tilde{S}_t} - {2\over 7}\Delta_{i,c} - \left(|\sum_{j\in \tilde{S}_t \setminus \hat{S}_t} (\theta_j^{k'}(t)  -  \hat{\mu}_j(t))  |\right) - \left(|\sum_{j\in \hat{S}_t \setminus \tilde{S}_t} (\theta_j^{k'}(t) -  \hat{\mu}_j(t))  | \right)\\
\nonumber&\ge &\Delta_{S^*,\tilde{S}_t} - {4\over 7}\Delta_{i,c}\\
\nonumber&>& {2\over 7}\Delta_{i,c},
\end{eqnarray}
where Eq. \eqref{Eq_3} is because of Eq. \eqref{Eq_2}.
This means that $\tilde{\Delta}^{k'}_t > {2\over 7}\Delta_{i,c}$, which contradicts with $\tilde{\Delta}_t^{k_t^*} \le {2\over 7}\Delta_{i,c}$ (since $k_t^* = \argmax_{k} \tilde{\Delta}^k_t$).

Case ii): $(i\in \hat{S}_t) \land (i\notin S^*)$ or $(i\in \tilde{S}_t) \land (i\in S^*)$. In this case, $i \in S^* \oplus \hat{S}_t$ and therefore $\Delta_{i,c} \le \Delta_{S^*,\hat{S}_t}$.

By event $\mathcal{E}_{2,c}$, we know that $\exists k, \sum_{j\in S^*} \theta_j^k(t) - \sum_{j\in \hat{S}_t} \theta_j^k(t) \ge \Delta_{S^*,\hat{S}_t}$. Hence $\tilde{\Delta}_t^k \ge \Delta_{S^*,\hat{S}_t}$. Moreover, since $k_t^* = \argmax_k \tilde{\Delta}_t^k$, we have that $\sum_{j\in \tilde{S}_t} \theta_j^{k_t^*}(t) - \sum_{j\in \hat{S}_t} \theta_j^{k_t^*}(t) \ge \Delta_{S^*,\hat{S}_t}$, which is the same as $\sum_{j\in \tilde{S}_t \setminus \hat{S}_t} \theta_j^{k_t^*}(t) - \sum_{j\in \hat{S}_t \setminus \tilde{S}_t} \theta_j^{k_t^*}(t) \ge \Delta_{S^*,\hat{S}_t}$. On the other hand, by event $\mathcal{E}_{1,c}$, we have that $|\sum_{j\in \tilde{S}_t \setminus \hat{S}_t} (\theta_j^{k_t^*}(t) - \hat{\mu}_j(t))| \le \sqrt{C(t)\sum_{j\in \tilde{S}_t \setminus \hat{S}_t}{2\over N_j(t)}L_2(t)} \le {\Delta_{i,c} \over 7}$, and similarly $|\sum_{j\in \hat{S}_t \setminus \tilde{S}_t} (\theta_j^{k_t^*}(t) - \hat{\mu}_j(t))| \le \sqrt{C(t)\sum_{j\in \hat{S}_t \setminus \tilde{S}_t}{2\over N_j(t)}L_2(t)} \le {\Delta_{i,c} \over 7}$ (recall that $N_i(t) \ge {98\width C(t)L_2(t) \over \Delta_{i,c}^2}$ and $i = \argmin_{j\in \hat{S}_t \oplus \tilde{S}_t} N_j(t)$). 

Therefore,
\begin{eqnarray*}
&&\sum_{j\in \tilde{S}_t \setminus \hat{S}_t}  \hat{\mu}_j(t) - \sum_{j\in \hat{S}_t \setminus \tilde{S}_t}  \hat{\mu}_j(t) \\
&=& \left(\sum_{j\in \tilde{S}_t \setminus \hat{S}_t} \theta_j^{k_t^*}(t) - \sum_{j\in \hat{S}_t \setminus \tilde{S}_t} \theta_j^{k_t^*}(t)\right) - \left(\sum_{j\in \tilde{S}_t \setminus \hat{S}_t} (\theta_j^{k_t^*}(t) - \hat{\mu}_j(t)) \right) + \left( \sum_{j\in \hat{S}_t \setminus \tilde{S}_t} (\theta_j^{k_t^*}(t) - \hat{\mu}_j(t)) \right)\\
&\ge& \left(\sum_{j\in \tilde{S}_t \setminus \hat{S}_t} \theta_j^{k_t^*}(t) - \sum_{j\in \hat{S}_t \setminus \tilde{S}_t} \theta_j^{k_t^*}(t)\right) - \left(|\sum_{j\in \tilde{S}_t \setminus \hat{S}_t} (\theta_j^{k_t^*}(t) - \hat{\mu}_j(t))| \right) - \left( |\sum_{j\in \hat{S}_t \setminus \tilde{S}_t} (\theta_j^{k_t^*}(t) - \hat{\mu}_j(t))| \right)\\
&\ge& \Delta_{S^*, \hat{S}_t} - {2\over 7}\Delta_{i,c} \\
&>& 0.
\end{eqnarray*}
This means that $\sum_{j\in \tilde{S}_t }  \hat{\mu}_j(t) - \sum_{j\in \hat{S}_t}  \hat{\mu}_j(t) > 0$, which contradicts with the fact that $\hat{S}_t = {\sf Oracle}(\bm{\hat{\mu}}(t))$.
\end{proof}

\section{How to Obtain Complexity Upper Bound by Eq. \eqref{Eq_1}}\label{Appendix_B}

Eq. \eqref{Eq_1} is restated below: 

\begin{equation*}
    Z \le O\left(H_{1,c}{\left(\log (|\mathcal{I}|Z) + \log{1\over \delta}\right)^2 \over \log{1\over q}}\right).
\end{equation*}

We can then use the following lemma to find an upper bound for complexity $Z$. 
\begin{lemma}\label{Lemma_Count}
    Given $K$ functions $f_1(x), \cdots, f_K(x)$ and $K$ positive values $X_1, \cdots, X_K$, if $\forall x \ge X_k, Kf_k(x) < x$ holds for all $1\le k \le K$, then for any $x\ge \sum_{k} X_k, \sum_k f_k(x) < x$.
\end{lemma}

\begin{proof}
    Since $X_1, \cdots, X_K$ are positive values, for any $x\ge \sum_{k} X_k$, we must have that $x \ge X_k$.
   Therefore $Kf_k(x) < x$, which implies that
    \begin{eqnarray*}
    \sum_{k} Kf_k(x) < \sum_k x.
    \end{eqnarray*}
    This is the same as $\sum_k f_k(x) < x$.
\end{proof}

To apply Lemma \ref{Lemma_Count}, we set $f_1(Z) = H_{1,c}{\log^2(|\mathcal{I}|Z)\over \log{1\over q}}$, $f_2(Z) = H_{1,c}{\log(|\mathcal{I}|Z)\log{1\over \delta}\over \log{1\over q}}$, and $f_3(Z) = H_{1,c}{\log^2{1\over \delta}\over \log{1\over q}}$.
After some basic calculations, we get that $X_1 \le c_1 H_{1,c}{\log^2(|\mathcal{I}|H_{1,c})\over \log{1\over q}}$, $X_2 \le c_{2,1}H_{1,c}{\log(|\mathcal{I}|H_{1,c})\log{1\over \delta}\over \log{1\over q}} + c_{2,2}H_{1,c}{\log^2{1\over \delta}\over \log{1\over q}}$ and $X_3 \le c_3H_{1,c}{\log^2{1\over \delta}\over \log{1\over q}}$. Here $c_1, c_{2,1}, c_{2,2}, c_3$ are universal constants. 





Then we know that for $Z \ge \Theta(H_{1,c}(\log{1\over \delta} + \log(|\mathcal{I}|H_{1,c})){\log{1\over \delta} \over \log{1\over q}} + H_{1,c}{\log^2(|\mathcal{I}|H_{1,c}) \over \log{1\over q}})$, $f_1(Z) + f_2(Z) + f_3(Z) < Z$ (by Lemma \ref{Lemma_Count}).
This contradicts with Eq. (\ref{Eq_1}).
Therefore, we know that
\begin{equation*}
Z = O\left(H_{1,c}{\left(\log (|\mathcal{I}|H_{1,c}) + \log{1\over \delta}\right)^2 \over \log{1\over q}}\right) = O\left(H_{1,c}\left(\log{1\over \delta} + \log(|\mathcal{I}|H_{1,c})\right){\log{1\over \delta} \over \log{1\over q}} + H_{1,c}{\log^2(|\mathcal{I}|H_{1,c}) \over \log{1\over q}}\right),
\end{equation*}
and this is the complexity upper bound in Theorem \ref{Theorem_Explore}.

\section{Discussions about the Optimal Algorithms for Combinatorial Pure Exploration}\label{Appendix_C}

\citet{chen2017nearly} prove that the complexity lower bound for combinatorial pure exploration is $\Omega(H_{0,c}\log{1\over \delta})$, where $H_{0,c}$ is the optimal value of the following convex program (here $\Delta_{S^*,S} = \sum_{i\in S^*} \mu_i - \sum_{i\in S} \mu_i$):
\begin{eqnarray*}
 \min && \sum_{i\in [m]} N_m\\
 \mathrm{s.t.} && \sum_{i\in S^* \oplus S} {1\over N_i} \le \Delta_{S^*, S}^2, \quad\forall S \in \mathcal{I}, S \ne S^*
\end{eqnarray*}

In other words, for any correct combinatorial pure exploration algorithm, $[{\E[N_1] \over \log{1\over \delta}}, {\E[N_2] \over \log{1\over \delta}}, \cdots, {\E[N_m] \over \log{1\over \delta}}]$ must be a feasible solution for the above convex program, where $\E[N_i]$ represents the expected number of pulls on base arm $i$.
%
We say base arm $i$ needs exploration the most at time $t$ if $\alpha N_i^* \log{1\over \delta} - N_i(t)$ is positive, where $\alpha$ is some universal constant and $N_i^*$ is the value of $N_i$ in the optimal solution $H_{0,c}$ (note that there may be several base arms that need exploration the most in one time step). 
%
%
By this definition, if we always pull a base arm that needs exploration the most, then the frequency of pulling each base arm converges to the optimal solution of $H_{0,c}$, which leads to an optimal complexity upper bound. 

However, the simple offline oracle used in TS-Explore (as described in Section \ref{Section_Model_CMAB}) is not enough to look for a base arm that needs exploration the most. In fact, both the existing optimal policies \citep{chen2017nearly,jourdan2021efficient} not only need to use this simple offline oracle, but also require some other mechanisms to explore detailed information about the combinatorial structure of the problem instance to look for a base arm that needs exploration the most.
The algorithms in \cite{chen2017nearly} need to record all the super arms in $\mathcal{I}$ with an empirical mean larger than some threshold $\eta$. This is one kind of information about the combinatorial structure that can help to find out a base arm that needs exploration the most. 
Nevertheless, the authors only provide an efficient way to do this when the combinatorial structure satisfies some specific constraints.
In the most general setting, the algorithms in \cite{chen2017nearly} must pay an exponential time cost for collecting the detailed information. 
As for \citep{jourdan2021efficient}, the best-response oracle used by the $\lambda$-player needs to return a super arm within set $N(S_t)$ that has the shortest distance to a fixed target. 
%
Here $S_t$ is a super arm, $N(S_t)$ represents the set of super arms whose cells’ boundaries intersect the boundary of the cell of $S_t$, and the cell of a super arm $S_t$ is defined as all the possible parameter sets $[\theta_1, \theta_2, \cdots, \theta_m]$ in which $S_t$ is the optimal super arm. 
%
This is another kind of information about the combinatorial structure that can help to find out a base arm that needs exploration the most.
%
Nevertheless, this best-response oracle also has an exponential time cost (which is scaled with $|N(S_t)|$).
%
%
%
By using these exponential time cost mechanisms, the optimal algorithms \citep{chen2017nearly,jourdan2021efficient} can find out a base arm that needs exploration the most, which is critical to achieving the optimal complexity upper bound. 
%
%
%

In this paper, to make TS-Explore efficient in the most general setting, we only use the simple offline oracle in our algorithm and our mechanism can only inform us of one of the violated constraints in the optimization problem (if all the constraints are not violated, TS-Explore will output the correct optimal arm). This means that we know nothing about the combinatorial structure of the problem instance. 
%
%
Therefore, the best thing we can do is to treat all the base arms in the violated constraint equally, i.e., we choose to pull the base arm (in the violated constraint) with the smallest number of pulls. 
This leads to a complexity upper bound of $O(H_{1,c}\log{1\over \delta})$. 
If we want 
TS-Explore to achieve an optimal complexity upper bound $O(H_{0,c}\log{1\over \delta})$, then we need to 
know which base arm 
needs exploration the most, e.g., by applying a powerful offline oracle that takes the empirical means and random samples as input and outputs a base arm that needs exploration the most. 
How to design such offline oracles 
and how to implement them efficiently is one of our future research topics.

\section{Concentration Inequality of Sub-Gaussian Random Variables}\label{Appendix_D}

\begin{fact}
If $X$ is zero-mean and $\sigma^2$ sub-Gaussian, then
\begin{equation*}
    \Pr[|X| > \epsilon] \le 2\exp(-{\epsilon^2 \over 2\sigma^2}).
\end{equation*}
\end{fact}




\end{document}